\documentclass[11pt]{article}
\usepackage{amsmath,amsfonts,amsthm,enumitem,hyperref,cancel,color,geometry,ulem}

\newcommand{\green}[1] {\textcolor{green}{#1}}
\newcommand{\Cs}{C}
\renewcommand{\th}{\theta}
\newcommand{\tth}{\widetilde{\th}}
\newcommand{\reals}{\mathbb{R}}
\newcommand{\project}[1]{[#1]_G}

\newcommand{\stepsize}{\epsilon}
\newcommand{\model}{\th}
\newcommand{\nablamod}{\nabla}
\newcommand{\hnC}{\nablamod \Cs}
\newcommand{\hnsqC}{\nablamod^2 \Cs}
\newcommand{\sn}{\xi}
\newcommand{\stepo}{c_0}
\newcommand{\E}{\mathbb{E}}
\newcommand{\lip}{L}
\newcommand{\p}{\prime}
\newcommand{\Ds}{D}
\def\para#1{\vskip .1 \baselineskip \noindent{\bf #1}}

\newtheorem{theorem}{Theorem}

\newtheorem{corollary}[theorem]{Corollary}
\newtheorem{lemma}[theorem]{Lemma}
\newtheorem{rem}[theorem]{Remark}
\newtheorem{prop}[theorem]{Proposition}
\newcommand{\ole}{\overset{\text{defn}}{=}}
\newcommand{\regret}{\operatorname{Regret}}
\renewcommand{\(}		{\left(}
  \renewcommand{\)}		{\right)}
\newcommand{\wdt}{\widetilde}
\newcommand{\wdh}{\widehat}
\newcommand{\termA}{T_1}
\newcommand{\termB}{T_2}
\newcommand{\termAa}{T_{11}}
\newcommand{\termAb}{T_{12}}
\newcommand{\nn}{\nonumber}

\newcommand{\eeq}{\end{equation}}
\newcommand{\bed}{\begin{displaymath}}
\newcommand{\eed}{\end{displaymath}}
\newcommand{\bea}{\bed\begin{array}{rl}}
\newcommand{\eea}{\end{array}\eed}
\newcommand{\lbar}{\overline}
\newcommand{\rr}{\mathbb R}
\newcommand{\cd}{(\cdot)}
\newcommand{\PP}{\mathbb P}
\newcommand{\be}{\beta}
\def\({\Big(}
\def\){\Big)}
\def\al{\alpha}

\newcommand{\barray}{\begin{array}{ll}}
\newcommand{\earray}{\end{array}}
\newcommand{\ad}{&\!\!\!\disp}
\newcommand{\aad}{&\disp}
\newcommand{\disp}{\displaystyle}

\newenvironment{myassumptions}{%
   \begin{description}[style=multiline, leftmargin = 18pt, align=left,font=\normalfont]%
}{%
   \end{description}%
}

 \makeatletter
 \def\nl#1#2{\begingroup
     \textbf{#2}%
     \def\@currentlabel{#2}%
     \phantomsection\label{#1}\endgroup
 }
 \makeatother

 \title{Finite Sample and Large Deviations Analysis  of Stochastic Gradient Algorithm with Correlated Noise}

\author{George Yin\footnote{Department of Mathematics, University of Connecticut, gang\_george.yin@uconn.edu}  \and Vikram Krishnamurthy\footnote{School of Electrical \& Computer Engineering, Cornell University  vikramk@cornell.edu}}

\begin{document}
\maketitle

\begin{abstract}
We analyze the finite sample regret of a decreasing step size stochastic gradient algorithm. We assume correlated noise and use a perturbed Lyapunov function as a systematic approach for the analysis. Finally we analyze the escape time of the iterates using large deviations theory.
\end{abstract}

\section{Introduction}
This paper focuses on  finite sample analysis for stochastic gradient algorithms. The motivation stems from a vast varieties of applications. In particular, the recent advances on stochastic optimization in conjunction with machine learning have opened up new domains. A particular emphasis of the learning community requires us taking  a careful look at of the finite sample analysis.
Well, it is well known that stochastic gradient algorithms or stochastic approximation algorithms are normally concentrated on dealing with asymptotic properties of the recursive algorithms. However, the learning community placed more effort for carrying out analysis of finite sample properties of the recursive algorithms; see for example, ... and references therein.

With the aforementioned motivation,
we  focus on the finite sample  analysis of  the mean square error and regret of the decreasing step size stochastic gradient algorithms.
While extensive effort has been on treating independent and identically distributed random disturbances, one almost always needs to face random noise and effect that correlated stochastic sequences must be taken into consideration.
To handle  correlated noise, we use the methods of perturbed Lyapunov function, as a systematic approach for  the analysis.
The analysis below shows  that the mean square error of the stochastic gradient algorithm after $n$ steps is $O(1/n)$. So the regret is logarithmic since $\sum_{k=1}^n O( 1/k ) = O(\log n)$.

In this paper,
we assume that  the expected cost  (objective function) $\Cs(\th)$ is convex and continuously differentiable in $\th \in \reals^p$. Denote the global minimizer of $\Cs(\th)$
by $\th^* \in \reals^p$.
Consider a decreasing step size  stochastic gradient algorithm of the form
\begin{equation}
  \label{eq:finite_proj}
    \th_{k+1} = \project{\th_k - \stepsize_k\,\hnC(\model_k,\sn_k)}, \qquad k=0,1,\ldots,
  \end{equation}
  where $\project{\cdot}$ denotes  projection of the estimate $\th_{k}$ to a  compact  set $G$.
The decreasing step size sequence is chosen as $\stepsize_k = \stepo/(k+1)$. For convenience, we assume $\stepo = 1$.
Assume throughout the paper, $\th^* \in G^\circ$, the interior of $G$. This is not a restriction since we can always choose $G$ to be large enough to have $\th^*$ be in the interior.

\section{Assumptions}
To carry out the analysis, we will use the following assumptions.
Note that we are mainly working with smooth functions. The key point is to work with finite samples, not to find weakest conditions possible. Thus, some of the assumptions can indeed be weakened. However, the current conditions will help us to get the analysis in a strict forward way without much technical details.

\begin{myassumptions}

  \item[\nl{E1}{(A1)}]
   The objective function  $\Cs(\cdot)$  is convex and twice continuously differentiable with respect to $\th \in \reals^p$.
    For each $\sn$,
     the first and the second partial derivatives with respect to $\th$ of
     $\Cs(\cdot,\sn)$, namely, $\nabla C(\cdot,\xi)$ and $\nabla^2 C(\cdot,\xi)$ exist and are continuous,
$\|\nablamod \Cs(0,\sn)\| \leq \tilde{K}_0$  w.p.1, and
  $\|\nablamod \Cs(\th,\sn) - \nablamod \Cs(0,\sn) \| \leq \bar{\lip} \|\th\|$ for a positive constant  $\bar{\lip}$.

  \item[\nl{E2}{(A2)}]   The noise  $\{\sn_k\}$ is a bounded  stationary uniform mixing sequence such that for each $\th$,
      \begin{itemize}
\item[(a)]
$\Cs(\th) = \E\{ \Cs(\th,\sn_k)\}$, 
\item[(b)] $\{ \nablamod \Cs(\th)-
\hnC(\th,\sn_k) \}$ is  a stationary mixing sequence with mixing rate
$\psi_k$ such that
\begin{equation}
  \label{eq:mixing-1}\barray \ad
\sum^\infty_{k=1} \psi_k <\infty,\\
 \ad
 \| \E_n \{ \nablamod \Cs(\th)-
\hnC(\th,\sn_k) \}\| \leq \psi_{k-n} \ \hbox{ for } \ k\ge n, \earray\end{equation}
\item[(c)] $ \{ \nablamod^2 \Cs(\th)-
\hnsqC(\th,\sn_k) \}$ is  stationary mixing sequence with mixing rate
$\bar{\psi}_k$ such that
\begin{equation}
  \label{eq:hess_mix}\barray\ad
 \sum^\infty_{k=1} \bar{\psi}_k <\infty,\\ \ad  \| \E_n \{ \nablamod^2 \Cs(\th)-
\hnsqC(\th,\sn_k) \}\| \leq \bar{\psi}_{k-n} \ \hbox{ for }\ k\ge n.\earray
\end{equation}
In the above
$\E_n$ denotes  conditional expectation w.r.t. the $\sigma$-algebra generated by  $\{\th_0, \sn_j : j < n\}$.
\end{itemize}

\item[\nl{E3}{(A3)}] There exists a nonnegative and
 twice continuously differentiable Lyapunov function $V(\cdot) :  {\mathbb R} ^p \mapsto {\mathbb R}$
  satisfying
  $V(\th) \to \infty$ as $\|\th\|\to \infty$ and
  $\nablamod V^\p (\th) \, \nablamod \Cs(\th) > 0$ for any $\th \not = \th^*$.

\item[\nl{E4}{(A4)}] The objective  function $\Cs(\th)$ is locally quadratic. That is, there is a symmetric positive definite matrix $B$, whose smallest eigenvalue is bounded by $\lambda >1$ such that
  \begin{equation}
    \label{eq:C-form}
 \Cs(\th)= \frac{1}{2}(\th - \th^*)^\p B (\th-\th^*) + \Ds (\th),
\end{equation}
such that $ \|\nabla \Ds (\th)\| \le K_2 \|\th -\th^*\|^{1+\alpha}$
for some constants $K_2>0$ and  $\alpha>0$.
\end{myassumptions}

\begin{rem}\label{rem:cond}{\rm
We comment on the conditions briefly as follows.
\begin{itemize}
\item[(a)]  Note that $\bar L$ in (A1) depends on $\xi$. In addition,
$\nabla C(0,\xi)$ generally is not 0.
\item[(b)] In (A2), the mixing rates $\psi_k$ and $\bar \psi_k$ are taken to be  positive real numbers. This follows from the classical treatment of Billingsley \cite{Bill}.  However as pointed out in \cite{EK}, random $\psi_k$ can be used.
\item[(c)]  $\th_0$ can be random. Throughout this paper, for simplicity, we often assume $\th_0$ to be a non-random quantity.
 \item[(d)]
The sequence of estimates  $\{\th_k\}$ is bounded
w.p.1
uniformly in $k$.
That is,
\begin{equation}
\sup_{k} \|\th_k\|\le K_0 \ \hbox{ w.p.1 for some } \ K_0>0,   \label{K0}
\end{equation}

which is a direct consequence of the project algorithm because  $\th_k \in G$ and $G$ is a compact set.
\end{itemize}
}\end{rem}


\section{Main result}

The proof of the following result is essentially in \cite{Yin91}.
A crucial step is to show that  $$ \sum^\infty_{k=1} {1\over k}  [\nabla C(\th)- \nabla C(\th,\xi_k)] \ \hbox{ converges w.p.1.}$$ The verbatim details can be found in the aforementioned reference, in particular, Theorem 3.1. For further reading and more general setup, the reader is referred to
\cite[Chpater 6]{KY03} for more details.

\begin{prop}\label{prop:conv} Under conditions {\rm (A1)-(A3)},
$\th_k\to \th^*$ w.p.1 as $k\to \infty$.
\end{prop}


For our subsequent study, the following result is useful.

\begin{prop} \label{res:finite_as}
Under {\rm\ref{E1}-\ref{E4}}, for any $\gamma \in [0,1/2)$, $\|\th_n - \th^*\| = o(n^{-\gamma})$ w.p.1.
\end{prop}

\para{Proof.} For a proof of the result, we refer to Theorem 3.1.1 (pp. 101-103) of \cite{Chen02}.

\begin{rem}\label{ga-bd}{\rm
In view of Proposition \ref{res:finite_as},   $ n^{\gamma}\|\th_n - \th^*\| \to 0$  w.p.1.
Then we can get an even  coaser bound in that
there is a positive integer $\wdt \kappa_+$ such that for all $n \ge \wdt \kappa_+$,
\begin{equation} \label{eq:gamma} n^{\gamma}\|\th_n - \th^*\| \le K \hbox{ for some } \ K >0.\end{equation}
Here and hereafter, we use $K$ as a generic positive constant with the understanding of $KK=K$ and $K+K=K$ in an appropriate sense.
}\end{rem}

Let us specify the various constants.
\begin{enumerate}
\item By \eqref{K0} and the triangle inequality,  $\|\th_n - \th^*\| \leq 2 K_0$. So we can choose $K= 2K_0$.
\item
Result~\ref{res:finite_as} implies
$\|\th_n - \th^*\| \leq K_2$ w.p.1 for  any positive constant $K_2$ that we choose, providing  the
sample size  $n >  (K/K_2)^{1/\gamma}= (2K_0/K_2)^{1/\gamma}$. Specifically, we will choose $K_2 = (\lambda_0/K_\Ds)^{1/\alpha}$ where $\lambda_0 \in (0,\lambda-1)$, and $K_\Ds$, $\alpha$, $\lambda$ are defined in \ref{E4}.
\item The outcome of steps 1 and 2 is: By Result~\ref{res:finite_as}, choosing  the sample size
\begin{equation}
  \label{eq:kappa1}
 n >  \kappa_1 \ole (2K_0/K_2)^{1/\gamma}, \text{ where } K_2 = (\lambda_0/K_\Ds)^{1/\alpha} \implies
 K_\Ds \|\th_n - \th^*\|^\alpha \leq \lambda_0 \text{ w.p.1. }
\end{equation}
\item
Next, by \ref{E2}, we choose integer  $\kappa_2 $ in terms of the mixing coefficients such that
\begin{equation}
  \kappa_2 = \inf \{ n \geq 1:
  \sum^\infty_{j=n} \psi_j \le 1, \;  \sum^\infty_{j=n} \bar{\psi}_j \le 1 \}.
\end{equation}
\item
With $\kappa_1$ and $\kappa_2$ defined above, let
\begin{equation}
 \kappa_+ = \max\{\kappa_1,\kappa_2\}.
\label{eq:k+}
\end{equation}
\end{enumerate}
Below we will work with time $n\ge \kappa_+$. The  main finite sample  result is the following.

\begin{theorem}
  \label{thm:msedec}
  Assume  {\rm\ref{E1}-\ref{E4}}.
  Then for $n\ge \kappa_+$ defined in \eqref{eq:k+}, the mean square error of the decreasing step size stochastic gradient  algorithm satisfies
  \begin{equation}
  \label{eq:mse_dec}
  \E  \|\th_n-\th^*\|^2\le \frac{K}{n}, \quad \text{ where K is a postive constant. }
\end{equation}
\end{theorem}

The mean square error yields the regret of the stochastic gradient algorithm.
Next, define the regret over the time interval $k=\kappa_+,\ldots, n$ as
\begin{equation}
\label{eq:regret_sa}
\regret_n =\sum^n_{k=\kappa_+} [\Cs(\th_k)- \Cs(\th^*)].
\end{equation}
Since $\Cs$ is continuously differentiable, clearly $\Cs(\th)- \Cs(\th^*) \le \lip\, \|\th - \th^*\|^2$ for positive constant $L$. We have the following simple corollary to Theorem~\ref{thm:msedec} that establishes logarithmic regret.

\begin{corollary}
Assume  {\rm \ref{E1}-\ref{E4}}. Then for $n \geq \kappa_+$,
the expected regret of the decreasing step size stochastic gradient algorithm is $$ \E\{ \regret_n\} \leq K\,\lip \log n.$$
\end{corollary}

\begin{proof}
 \begin{align*}
   \E \{ \regret_n \} &= \sum^n_{k=\kappa_+} \E \{\Cs(\th_k)- \Cs(\th^*)\}
                        \leq  \lip \sum^n_{k=\kappa_+} \E \|\th_k-\th^*\|^2 \leq
                        K \lip    \sum_{k=1}^n \frac{1}{k} .
 \end{align*}
 \end{proof}

 The mean square estimation error for stochastic gradient algorithms (Theorem~\ref{thm:msedec})  has been analyzed extensively over  50 years; see \cite{BMP90,KY03} for  general results. Going from  mean square error  to regret is elementary as shown in the corollary above.

 \section{Proof of Theorem~\ref{thm:msedec}}
Choose
 $V(\th)=\th' \th/2$.  Denote  the estimation error as $\tth_n = \th_n - \th^*$. Then
$$V(\tth_{n+1}) - V(\tth_n) = \frac{1}{n} \tth_n^\p \big[\nabla\Cs(\th_n)- \hnC(\model_n,\sn_n)\big] -\frac{1}{n} \tth_n^\p\nabla\Cs(\th_n) +
\frac{1}{2n^2} \,\|\hnC(\model_n,\sn_n) \|^2 .
$$
By virtue of
\ref{E1},
\begin{equation}\label{e1}
\barray
\|\hnC(\sn_n,\model_n) \|\ad
 \le \| \hnC(\model_n, \sn_n)-\hnC(0,\sn_n)\|+
\|\hnC(0,\sn_n)\| \\
\ad
\le \bar L(\xi_n)\|\th_n\| +\wdt K_0 ,\earray\end{equation}
so $$\frac{1}{2n^2} \,\|\hnC(\model_n, \sn_n) \|^2\le
[\bar{\lip}^2(\xi_n) K_0 +\wdt K_0]^2/(2n^2)
\ole K_3/n^2.$$
Using the local-quadratic assumption \ref{E4}, the second to the last term is bounded by
\begin{equation}
  \begin{split}
    -\frac{1}{n} \tth_n^\p\nabla\Cs(\th_n)
    &= -\frac{1}{n} \tth_n^\p [B \tth_n + \nabla \Ds(\th_n)]\\
&\overset{\text{(a)}}{\leq} -\frac{1}{n} \lambda \tth_n^\p \tth_n + \frac{1}{n} |\tth_n^\p\,\nabla\Ds(\th_n)| \\
    &\overset{\text{(b)}}{\leq}  -\frac{1}{n} \lambda \tth_n^\p \tth_n + \frac{K_D}{n} \tth_n^\p\,\tth_n\,\|\tth_n\|^\alpha\\
    &
     \overset{\text{(c)}}{\leq} -\frac{1}{n} (\lambda - \lambda_0) \,V(\tth_n).
  \end{split}
\end{equation}
(a) holds since by \ref{E4}, $\lambda> 1$ is the smallest eigenvalue of $B$. (b) follows from the bound on $|\nabla \Ds(\th)|$ in \ref{E4}. Finally, (c) is a consequence of   
Proposition \ref{res:finite_as}. In fact, by Proposition \ref{res:finite_as},
in particular \eqref{eq:gamma}, for all $n\ge \kappa_+$, and
 for almost all $\omega$ and some $\wdh K>0$, $ [n^\gamma \|\th_n \|]^\alpha \le \wdh K$,
and as a result,   $\|\th_n \|^\alpha \le \wdh K/ n^{\gamma \alpha}$.
As a result, we have $K_2 \|\tth_n\|^\alpha \leq \lambda_0$ w.p.1 for any positive constant $\lambda_0$. We choose $\lambda_0$ small enough so that $\lambda_0 \in (0,\lambda-1)$.
So set $\lambda_1 \ole = \lambda - \lambda_0 > 1$.
Recall that $\E_n$ is the conditional expectation w.r.t.\   $\{\th_0, \sn_j : j < n\}$. Then
\begin{equation}
  \label{eq:en1}
  \E_n V(\tth_{n+1})- V(\tth_n) \le -{\lambda_1 \over n} V(\tth_n) +
  {1\over n}\E_n\big\{\tth'_n [\nablamod C(\th_n) -  \hnC(\th_n, \sn_n) ]\big\}+{K_3\over n^2}.
\end{equation}

{\bf
Perturbed Lyapunov Function Approach  for  Correlated Noise}.
We now consider the case where the noise is correlated and \ref{E2} holds.
We use the perturbed Lyapunov function approach to tackle the problematic  term $\E_n\{\cdot\}$ in the RHS of~\eqref{eq:en1}.

The main idea  is as follows: Define the perturbed  Lyapunov function
\begin{equation}
  \label{eq:perturbedd}
  \begin{split}
  W(\wdt \th,n)&= V(\wdt \th) + V_1(\wdt \th,n)  \\
\text{ where }  V_1(\wdt \th,n)&= \sum^{\infty}_{k=n} \frac{1}{k} \E_n\big\{\tth^\p [\nablamod C(\th) -  \hnC(\th,\sn_k) ]\big\}.
  \end{split}
\end{equation}
We will show that the perturbation $V_1(\wdt \th,n)$
is 
 cancels the
second   term on the right-hand side of
\eqref{eq:en1}.
Specifically, the perturbed Lyapunov function satisfies  the following two desirable  properties:

\begin{itemize}
\item
{\em Property 1}.
$V_1$ is a small perturbation in magnitude compared to $V(\tth)$  in that
\begin{equation}
  \label{eq:v1-smalld}
|V_1(\wdt \th, n)|
\le \frac{1}{n} ( V(\wdt \th)+1).
\end{equation}
Property 1 is easy to show.  Indeed, since $\{\th_n\}$ is bounded, then  using \ref{E2} we have
\begin{equation}
  \label{eq:v1-bd}
  \begin{split}
    | V_1(\tth,n)| &\le {1\over n} \| \sum^\infty_{k=n}  \tth' \big[ \E_n \{\nablamod C(\th)-
                      \hnC(\th,\sn_k) \}\big] \| \\
&\le |\tth| {1\over n} \sum^\infty_{k=n}\psi_{k-n} \le
{1\over n} (V(\tth)+1) .
  \end{split}
\end{equation}

\item
{\em Property 2}.
 The perturbed
Lyapunov function $W$ defined in~\eqref{eq:perturbedd} satisfies
\begin{equation}
  \label{eq:perturbed_requirementd}
  \hspace{-1cm}  \E_n W(\wdt \th_{n+1},n+1)-W(\wdt \th_n,n)
  \le
-{\lambda_1 \over n} W(\tth_n,n)
  +{\bar{K}\over n^2}.
\end{equation}
\end{itemize}


\begin{lemma}\label{lem:est-bas} Suppose $a> 1 $ and $b$ is a positive constant. Then
\begin{equation}
    \label{eq:scalar_rec}
    x_{n+1} \leq \big(1-\frac{a}{n} \big)x_n + \frac{b}{n^2} , \quad x_1 \geq 0
  \end{equation}
  implies
$ x_n \leq c/n$ for positive constant  $c \geq \max\{x_1,b/(a-1)\}$.
\end{lemma}

\begin{proof} (By induction).  Choosing $c \geq \max\{x_1,b/(a-1)\}$ accounts for $x_1$. Assume $x_n \leq c/n$. Then~\eqref{eq:scalar_rec} yields $x_{n+1} \leq c/n - (ac-b)/n^2$. Thus to show  $x_{n+1} \leq c/(n+1) = c/n - c/(n(n+1))$, it is sufficient that $ac - b \geq c n /(n+1)$. This holds if  $ac -b \geq c$, i.e.,
$  c \geq b/(a-1)$ since $a>1$.
\end{proof}

As a division of labor, we first assume
Property 2 holds.
By virtue of Lemma \ref{lem:est-bas},
$\E W(\wdt \th_{n+1},n+1) \leq  { K}/n$.
 Then from~\eqref{eq:perturbedd} and \eqref{eq:v1-smalld}, since the perturbations are small,
 $V$ also satisfies~ $\E V(\wdt \th_{n+1}) \leq  {K}/n$.  This completes the proof for the correlated noise case.  \hfill \qedsymbol{}

 We will prove Property 2 in what follows.

 {\em Proof of Property 2}. It only remains to prove  Property 2~\eqref{eq:perturbed_requirementd}. By definition
 \begin{equation}
   \label{eq:W-dec}
   \E_n W(\wdt \th_{n+1},n+1)-W(\wdt \th_n,n)
= \E_n V(\wdt \th_{n+1}) - V(\wdt \th_n)  +  \E_n V_1 (\wdt \th_{n+1},n+1)-  V_1(\wdt \th_n,n) .
 \end{equation}
 $ \E_n V(\wdt \th_{n+1}) - V(\wdt \th_n) $ was bounded in~\eqref{eq:en1}.
We now show   $\E_n V_1 (\wdt \th_{n+1},n+1)-  V_1(\wdt \th_n,n) $ cancels
the problematic term $\E_n\{\cdot\}$ in~\eqref{eq:en1} and has an additional small $O(1/n^2)$ term:
\begin{multline}
  \label{eq:prop2}
\E_n V_1(\tth_{n+1},{n+1})-V_1(\tth_n,n)
\\  = \underbrace{\E_n  [V_1(\tth_{n+1},{n+1})-V_1(\tth_{n},n+1)] }_{\termA=\termAa + \termAb \text{ defined in~\eqref{eq:v1-b}}}  +
\underbrace{\E_n  [V_1(\tth_{n},{n+1})-V_1(\tth_{n},n)]}_{\termB \text{ defined in~\eqref{eq:v1_2nd}}} .
  \end{multline}
  From the definition of $V_1$,
\begin{equation}
  \label{eq:v1_2nd}
\termB= \E_n  [V_1(\tth_{n},{n+1})-V_1(\tth_{n},n)]=
 - {1 \over n} \E_n \tth'_n [\nablamod C(\th_n)-  \hnC(\th_n,\sn_n) ].
\end{equation}
Notice that $\termB$  exactly cancels out the problematic $\E_n\{\cdot\}$ term in~\eqref{eq:en1}.

Next we show that $\termA$ in~\eqref{eq:prop2} is $O(1/n^2)$ and therefore small. Note
\begin{align}
    &\termA = \termAa + \termAb, \quad \text{where }
\termAa  = \sum^\infty_{k=n+1} {1 \over k} \E_n\{ [\tth_{n+1} -\tth_n]' [\nablamod C(\th_{n+1})-   \hnC(\th_{n+1}, \sn_k)] \}\nn \\
    & \termAb = \sum^\infty_{k=n+1} {1 \over k} \E_n \big\{\tth_n' [\nablamod C(\th_{n+1})-  \hnC(\th_{n+1},\sn_k)] -[\nablamod C(\th_{n})-  \hnC(\th_n,\sn_k)]\big\}.
        \label{eq:v1-b}
\end{align}
$\termAa$ in~\eqref{eq:v1-b} is  bounded as follows: Since
$$\tth_{n+1} - \tth_n = \th_{n+1} - \th_n = -\frac{1}{n} \hnC(\th_n,\sn_n),$$
\begin{equation}
  \begin{split}
|\termAa |  &\leq  {1 \over n ^2} \big\|   \hnC(\th_n,\sn_n)\big\|\, \,   \big\| \sum^\infty_{k=n+1}  \E_n  [\nablamod C(\th_{n+1})- \hnC(\th_{n+1},\sn_k)]\big\| \\
 & \overset{(a)}{\leq}  \frac{K_0 \bar{\lip} }{n^2} \sum_{k=n+1}^\infty \psi_{k-n}
   \overset{(b)}{\leq}  \frac{K_0 \bar{\lip} }{n^2}\; \text{ w.p.1. }
  \end{split} \label{eq:T11b}
\end{equation}
(a) follows  since   $\|\hnC(\th_{n}, \sn_k)\| \leq \bar{\lip} \th_{n}$ by \ref{E1},  $\|\th_{n}\| \leq K_0$,
and  applying mixing assumption \ref{E2}. (b) follows from~\eqref{eq:k+}.

Next, let us bound $\termAb$
in~\eqref{eq:v1-b}. This can be written as
$$|\termAb| \leq \| \sum^\infty_{k=n+1} \frac{1}{k} \E_n \{\tth_n^\p [\wdt f(\th_{n+1},\sn_k) - \wdt f(\th_n,\sn_k)] \}\| \;
\text{where }\wdt  f(\th_n,\sn_k) = \nablamod \Cs(\th_{n})- \hnC(\th_{n},\sn_k) $$
By first order Taylor series expansion,
$$\wdt f(\th_{n+1},\sn_k) -\wdt f(\th_n, \sn_k) = \nablamod \wdt f(\th_n^+, \sn_k)\, (\th_{n+1} - \th_n)  =
- \nablamod \wdt f(\th_n^+,\sn_k)\,\frac{1}{n} \hnC(\th_{n},\sn_k)
$$
where  $\th^+_n $ lies on  the line segment joining $\th_n$ and $\th_{n+1}$. Using this, we have
\begin{equation}
  \begin{split}
 |\termAb|  &\leq  \| \sum^\infty_{k=n+1} \frac{1}{k} \E_n \{\tth_n^\p [\wdt f(\th_{n+1},\sn_k) -\wdt f(\th_n,\sn_k)] \}\|
    \\
  & \leq
    \frac{1}{n} \| \tth_n\| \;\frac{1}{n}\,\| \E_n \{\sum_{k=n+1}^\infty
  \nablamod \wdt f(\th_n^+,\sn_k)\,  \hnC(\th_{n},\sn_n)\} \|
    \\
  & \leq
    \frac{1}{n^2} \| \tth_n\| \;\| \E_n \{\sum_{k=n+1}^\infty
    \nablamod \wdt f(\th_n^+,\sn_n)\} \| \,  \|\hnC(\th_{n},\sn_n) \| \\
    &
    \overset{\text{(a)}}{\leq} \frac{2\,K_0 \bar{L}}{n^2} \sum_{k=n+1}^\infty \bar{\psi}_{k-n}
    \leq \frac{2K_0 \bar{L}}{n^2} .  \label{eq:T12b}
  \end{split}
\end{equation}
(a) follows from~\eqref{eq:hess_mix} in \ref{E2} and
$\|\tth_n\| \leq \|\th_n\|+ \|\th^*\|\leq 2 K_0$.
Let us substitute the above results  into~\eqref{eq:W-dec}, repeated below for convenience:
$$  \E_n W(\wdt \th_{n+1},n+1)-W(\wdt \th_n,n) = \E_n V(\wdt \th_{n+1}) - V(\wdt \th_n)  +
\termAa +\termAb +\termB.
$$
Then substituting~\eqref{eq:en1}, \eqref{eq:T11b}, \eqref{eq:T12b}, \eqref{eq:v1_2nd} yields for positive constant $K$,
\begin{multline}
 \label{eq:finalW}
  \E_n W(\wdt \th_{n+1},n+1)-W(\wdt \th_n,n) \leq
-{\lambda_1 \over n} V(\tth_n) +
  \cancel{  {1\over n}\E_n\big\{\tth'_n [\nablamod C(\th_n) -  \hnC(\th_n,\sn_n) ]\big\} } 
    \\   - \cancel{{1 \over n} \E_n \big\{\tth'_n [\nablamod C(\th_n)-  \hnC(\th_n,\sn_n) ]\big\} }
         + \frac{K}{n^2}.
       \end{multline}
       Finally,  Property~1~\eqref{eq:v1-bd} implies
$$ V(\tth_n) \geq  \frac{n W(\tth,n)}{n+1} - \frac{1}{n+1}
\implies -{\lambda_1 \over n} V(\tth_n) \leq \frac{-\lambda_1 W(\tth,n)}{n+1} + \frac{\lambda_1}{n^2} .
$$
So  we can replace $ V(\wdt \th_n)$ in~\eqref{eq:finalW} with $W(\wdt \th_n,n) $ and maintain inequality.
Therefore Property~2, namely, \eqref{eq:perturbed_requirementd} holds. \hfill  \qedsymbol{}

\begin{rem}\label{rem:iid}
{\rm If the noise is an i.i.d. sequence, then much of the calculation can be simplified.
Suppose the noise  $\{\sn_n\}$ is i.i.d. Then the $\E_n\{\cdot\}$ term on the RHS of~\eqref{eq:en1} is zero. So
\begin{equation}
  \label{eq:en1u}
  \E_n V(\tth_{n+1})- V(\tth_n) \le -{\lambda_1 \over n} V(\tth_n)
  +{K_3\over n^2}.
\end{equation}
Therefore, no perturbations of the Lyapunov function is needed.
Then taking the expectation yields
$$  \E V(\tth_{n+1}) \leq  \( 1 - {\lambda_1\over n} \) \E V(\tth_n) +  \frac{K_3}{ n^2}.
$$
which implies $\E V(\tth_{n+1}) \leq  K/n$ by  Lemma \ref{lem:est-bas}.
}\end{rem}


\section{Escape Times}
In this section, we analyze the escape of the iterates from a small neighborhood of the minimizer $\th^*$.  The argument is along the line of large deviations. We shall use the techniques in \cite[Sections 6.9 and 6.10]{KY03}.
In fact, the discussion will be kept in a rather intuitive way so as to make the main idea clear.
We do not wish to go over all the technical details.

We show that if the iterates $\th_n$ gets close to $\th^*$ at large $n$,
they will stay in a small neighborhood of $\th^*$  for
a very long time. We quantify the ``very long time'' by showing the iterates will escape from a small neighborhood with a probability that is exponentially small. This, in fact, is an alternative way of the nowadays popular concentration probability estimates.

In view of the discussion in the last section,
we  rewrite the algorithm as
\begin{equation}\label{no-t-re}
\th_{n+1} = \th_n  + {1\over n}
 [\nabla C(\th_n)-\nabla c(\th_n,X_n)] -{1\over n} \nabla C(\th_n) .\eeq
Next, we define
\bea \ad t_0 = 0, \  t_{n+1} = t_n + {1\over n}, \ m(t)= \max\{n: t_n \le t\},\\
\ad  \lbar \th^0 (t)= \th_n \ \hbox{ for } \ t\in [t_n, t_{n+1}),\\
\ad \th^n(t)= \lbar \th^0 (t+t_n).\eea
In this section, our objective is to find escape probability from a neighborhood of $\th^*$. In a way, this is another approach to find the nowadays popular concentration probabilities.
We will begin the discussion in a general form, and then look into a specific form of the functions involved that enables us to estimate the escape probabilities.

\def\cc{{\cal C}}
The following is an approach given in our book \cite[Section 6.10]{KY03}.
Use $\cc[0,T]$ to denote the space of continuous functions on $[0,T]$ with initial data $\th$.
In \cite[Section 6.10]{KY03}, we worked out the general case by assuming that the following conditions hold.
Let $G_0$ and $G$ be a bounded neighborhood of $\th^*$ in $\rr^d$,
which is in the domain
of attraction of $\th^*$
 such that the ``translation''
$\th^*  + \lbar G = \{\th^* +y : y \in \lbar G\}$. The set $G_0$ can be arbitrarily small. There is
a real-valued function $H(\al,\psi)$ that is continuous in $(\al,\psi)$ in $G_0\times \lbar G$ and
whose $\al$-derivative is continuous on $G_0$ for each fixed $\psi \in \lbar G$ such that the
following limit holds: For any $T > 0$ and $\Delta > 0$ with $T$ being an integral
multiple of $\Delta$, and any functions $(\al\cd, \psi\cd)$  taking values in $(G_0, \lbar G)$ and
being constant on the intervals $[i\Delta; i\Delta + \Delta)$, $i\Delta < T$, under suitable conditions (see \cite[Sections 6.10]{KY03})
we have that
\begin{equation}\label{H-lim}\barray \disp\int^T_0 H(\al(s),\psi(s)) ds \ad \ge \limsup_{n,m\to \infty}
{\Delta \over m}  \log \E \exp \( \sum^{T/\Delta-1}_{i=0} \al'(i\Delta)\\
\aad \quad \times \sum^{im  + m -1}_{j=im} [\nabla C(\th^*  + \psi(i \Delta))- \nabla c(\th^* +\psi(i\Delta),X_{n+j})]\)\earray\eeq
exists for each $\al$ and each $\psi$.
Next, denote $H_1(\al,\psi,s)= e^{s} H(\al, \psi)$.
The reason for the $e^{s}$ can be found in \cite[two line below equation (10.5)]{KY03}.
Define the Legendre transformation
\begin{equation}\label{L-def} L(\beta,\psi,s)=\sup_{\al} [ \al'( \be - C(\th^*+\psi)) -H_1(\al,\psi,s)],\eeq
and define \begin{equation}\label{S-def}S(T,\psi)= \left\{ \barray \ad \int^T_0 L(\psi(u),\dot \psi(u),u) du \ \hbox{ if } \ \phi \ \hbox{ is absolutely continuous,} \\
\ad  \infty  \ \hbox{ otherwise.}\earray\right. \eeq
Then under smoothness condition of $C\cd$ and the mixing condition of the noise,
as in \cite[Theorem 2.1]{F82}, for each $A \subset \cc [0,T]$,  with $A^0$ and $\lbar A$ denoting the interior and closure of $A$,  respectively, we have
\bea \disp - \inf_{\phi \in A^0} S(T,\phi)\ad\le \liminf _{n}\lambda_n \log \PP_\th(\th^n\cd \in A)\\
\ad \le\limsup_n \log \PP_\th (\th^n\cd \in A) \\
\ad \le -\inf_{\phi \in \lbar A} S(T,\phi).\eea
 Define also $\tau^n_G$ as the first exit time of $\th^n\cd$ from $G$. That is,
$\tau^n_G = \inf \{t: \th^n(t) \not \in G\}$.
We are able to show that $\PP_\th (\tau^n_G \le T)$ is small.

To put the result in a more  concrete setting and easily to be visualized, we look at a specific case, which provides some more insight.
To this end, we assume the following assumptions hold. The assumed Gaussian distribution is mainly for simple representation of the moment generating functions and for better visualization.

\begin{myassumptions}
\item[\nl{E5}{(A5)}]
Assume
 \begin{equation}\label{c-def}\nabla c(\th,X) = \nabla C(\th)+ f_0(\th) X,\eeq
where $C\cd$ is the smooth function as specified before,  $f_0(\cdot): \rr^d \mapsto \rr^{d\times d}$ is a bounded and continuous matrix-valued function with $f_0(\th^*)\not =0$, and $\{X_n\}$ is a sequence of
Gaussian
stationary mixing process satisfying with mixing measure $\psi_n$ as in (A2),
$\E X_k=0$, and
$\E | X_k|^2 <\infty$.
\end{myassumptions}

Assume that conditions of the second moment estimates in the last section and (A5) hold. Then we can proceed with the analysis.
Note that the calculation in \eqref{H-lim} involves mainly the computation of $\log$ moment generating function.
It is easily seen that
the sequence $\{f_0(\th^*) X_n\}$ is a mixing sequence with mean $0$.
Denote $\xi_j= f_0(\th^*) X_j .$
Using the mixing property, we can show that for each $l$,
\begin{equation}\label{cov-d}\barray \disp \E {1\over m}\Big[ \sum^{l+m -1}_{j=l} \xi_j  \Big]
 \Big [\sum^{l+m -1}_{k=l} \xi_k\Big]'\ad
 \to \E \xi_0 \xi_0' + \sum^\infty_{j=0} \E \xi_j \xi_0'
+ \sum^\infty_{k=0} \E \xi_0\xi'_j \ \hbox{ as } \ m\to \infty\\
\ad =  f_0(\th^*)  [ R_0 + \sum^\infty_{j=0} R_j+ \sum^\infty_{j=0}   R'_j ] f'_0(\th^*)\\
\ad := f_0(\th^*) \lbar R f'_0(\th^*),
\earray\eeq
where $R_j =\E X_j X'_0$.
We realize that $\lbar R$ is just the limit covariance of the mixing process.
Now it is easily seen that the limit (in lieu of $\limsup$) exists in \eqref{H-lim}.
We have
\begin{equation}\label{h-def}\int^T_0 H(\al(s),\psi(s)) ds = \int^T_0 \al'(s) f_0(\th^*) \lbar R f'_0(\th^*) \al(s) ds.\eeq
Let
$B_\th$ be a set of continuous functions on $[0,T]$ taking values in the set
$\lbar G$,  and
  with initial value $\th$.
It follows that \cite[Theorem 10.3]{KY03} indicates that
$$\limsup_n{1\over n}\log \PP_\th^n\left\{\th^n\cd
\in B_\th\right\}\le -\inf_{\psi\in
\bar B_\th}\bar S(T,\psi).$$
Furthermore, as in
the following can be established.
 For sufficiently small $\mu$,
$\overline{N_\mu(\th^*)}\subset G.$
\cite[Theorems 10.3 and  10.4]{KY03} yield that
there are $h_0>0$ and $\mu_0>0$
(with
$\overline{N_{\mu_0}(\th^*)}\subset G$) such that for
$\mu\le\mu_0$, and for sufficiently large $n$, and all $\th\in N_{\nu(\mu)}(\th^*)$,
$$\PP^n_\th\left\{\th^n(t)\not\in G\hbox{ for some } 0\le t \le T
\hbox{ or } \th^n(T)\not\in N_{\nu(\mu)}(\th^*)\right\}\le
e^{-h_0 n}.$$
That is, the probability of the iterates exit from a small neighborhood of $\th^*$ is exponentially small. We can also show that for some $h_1>0$ and $\wdt K$ is close to $0.5$,
$$\E \tau^n_G \ge \wdt K T e^{h_1 n}.$$
That is, the expected time to exit from $G$ is ``infinitely'' long.

\begin{rem}\label{iid}{\rm
To get further insight, we look at an even simpler case with
$\nabla c(\th,X)$ given by \eqref{c-def}, in which $C\cd$ is the same as before,
$f_0(\cdot): \rr^d \mapsto \rr^{d\times d}$ is a bounded and continuous matrix-valued function, and $\{X_n\}$ is a sequence of independent and identically distributed random variables with
Gaussian distribution whose mean and covariance are $0$ vector and constant matrix $R_0$, respectively. Still denote $\xi_j= f_0(\th^*) X_j$.
Then it is readily seen that
$R_0= \E \xi_0 \xi'_0=\E\xi_k\xi'_k$ for any $k$.
Moreover, \eqref{h-def} simplifies to
$$\int^T_0 H(\al(s),\psi(s)) ds = \int^T_0 \al'(s) f_0(\th^*)  R_0 f'_0(\th^*) \al(s) ds.$$

}\end{rem}

\bibliographystyle{abbrvnat}
\bibliography{newreg}

\end{document}